\documentclass[final]{colt2022} % Include author names

\pdfoutput=1

\title[Generalization Bounds via Convex Analysis]{Generalization Bounds via Convex Analysis}
\usepackage{times}
\usepackage[utf8]{inputenc} % allow utf-8 input
\usepackage[T1]{fontenc}    % use 8-bit T1 fonts
\usepackage{hyperref}       % hyperlinks
\usepackage{url}            % simple URL typesetting
\usepackage{booktabs}       % professional-quality tables
\usepackage{amsfonts}       % blackboard math symbols
\usepackage{nicefrac}       % compact symbols for 1/2, etc.
\usepackage{microtype}      % microtypography
\usepackage{xspace}
\usepackage{amsmath}

\usepackage{mathtools}  
\usepackage{amsmath}
\usepackage{amssymb}
\usepackage{tabulary}
\usepackage{booktabs}

\usepackage{algorithm}
\usepackage{algpseudocode}
\usepackage{algpascal}
\usepackage{comment}
\usepackage{selectp}
\usepackage{selectp}

\newcommand{\dd}{\mathrm{d}}

\newcommand{\F}{\mathcal{F}}
\newcommand{\real}{\mathbb{R}}
\newcommand{\Dw}{\mathcal{D}}
\newcommand{\Sw}{\mathcal{S}}
\newcommand{\Pw}{\mathcal{P}}

\newcommand{\Ww}{\mathcal{W}}

\newcommand{\DD}[2]{\mathcal{D}\pa{#1\middle\|#2}}
\newcommand{\DDh}[2]{\mathcal{D}_{\bH}\pa{#1\middle\|#2}}

\newcommand{\DDsigma}[2]{\mathcal{D}_{\sigma}\pa{#1\middle\|#2}}
\newcommand{\DDKL}[2]{\mathcal{D}_{\mathrm{KL}}\pa{#1\middle\|#2}}
\newcommand{\DDPhi}[2]{\mathcal{B}_\Phi\pa{#1\middle\|#2}}

\newcommand{\DDchi}[2]{\mathcal{D}_{\chi^2}\pa{#1\middle\|#2}}

\newcommand{\OO}{\mathcal{O}}

\newcommand{\PP}[1]{\mathbb{P}\left[#1\right]}

\newcommand{\EE}[1]{\mathbb{E}\left[#1\right]}

\newcommand{\EES}[1]{\mathbb{E}_S\left[#1\right]}
\newcommand{\EESb}[1]{\mathbb{E}_S\bigl[#1\bigr]}
\newcommand{\EESB}[1]{\mathbb{E}_S\Bigl[#1\Bigr]}
\newcommand{\EEZ}[1]{\mathbb{E}_Z\left[#1\right]}

\newcommand{\EEs}[2]{\mathbb{E}_{#2}\left[#1\right]}

\newcommand{\EEcc}[2]{\mathbb{E}\left[\left.#1\right|#2\right]}

\def\argmax{\mathop{\mbox{ arg\,max}}}
\newcommand{\ra}{\rightarrow}

\newcommand{\iprod}[2]{\left\langle#1,#2\right\rangle}

\newcommand{\biprod}[2]{\bigl\langle#1,#2\bigr\rangle}

\newcommand{\norm}[1]{\left\|#1\right\|}
\newcommand{\bnorm}[1]{\bigl\|#1\bigr\|}

\newcommand{\twonorm}[1]{\norm{#1}_2}
\newcommand{\infnorm}[1]{\norm{#1}_\infty}
\newcommand{\tvnorm}[1]{\norm{#1}_{\mathrm{TV}}}

\newcommand{\ev}[1]{\left\{#1\right\}}
\newcommand{\abs}[1]{\left|#1\right|}
\newcommand{\babs}[1]{\bigl|#1\bigr|}
\newcommand{\pa}[1]{\left(#1\right)}
\newcommand{\bpa}[1]{\bigl(#1\bigr)}

\newcommand{\wt}{\widetilde}

\newcommand{\bloss}{\overline{\ell}}

\newcommand{\loss}{\ell}
\newcommand{\gen}{\textup{gen}}

\newcommand{\qed}{\hfill\BlackBox\\[2mm]}

\newcommand{\alg}{\mathcal{A}}
\newcommand{\Zw}{\mathcal{Z}}
\newcommand{\bL}{\overline{L}}
\newcommand{\bH}{h}

\usepackage[normalem]{ulem}
\usepackage{enumitem}

\author[Lugosi and Neu]{\Name[{G\'abor~Lugosi}]{G\'abor Lugosi} \Email{gabor.lugosi@gmail.com}\\
 \addr ICREA and Universitat Pompeu Fabra, Barcelona, Spain
 \AND
 \Name[{Gergely~Neu}]{Gergely Neu} \Email{gergely.neu@gmail.com}\\
 \addr Universitat Pompeu Fabra, Barcelona, Spain
 }

 \firstpageno{3524}
 
\begin{document}

\maketitle

\begin{abstract}
Since the celebrated works of \citet{RZ16,RZ19} and \citet{XR17}, it has been well known that the generalization error 
of supervised learning algorithms can be bounded in terms of the mutual information between their input and the output, 
given that the loss of any fixed hypothesis has a subgaussian tail. In this work, we generalize this  result 
beyond the standard choice of Shannon's mutual information to measure the dependence between the input and the output. 
Our main result shows that it is indeed possible to replace the mutual information by any strongly convex function 
of the joint input-output distribution, with the subgaussianity condition on the losses replaced by a bound on an 
appropriately chosen norm capturing the geometry of the dependence measure. This allows us to derive a range of 
generalization bounds that are either entirely new or strengthen previously known ones. Examples include bounds 
stated in terms of $p$-norm divergences and the Wasserstein-2 distance, which are respectively applicable for 
heavy-tailed loss distributions and highly smooth loss functions. Our analysis is entirely based on elementary tools from convex analysis by tracking the growth of a potential function associated with the dependence measure and the loss function.
\end{abstract}


\section{Introduction}
We study the standard model of supervised learning where we are given a set $S$ of $n$ i.i.d.~data points 
$S_n = \ev{Z_1,\dots,Z_n}$ drawn from a distribution $\mu$ and consider a learning algorithm that maps this data set to 
an 
output $W_n = \alg(S_n)$ in a potentially randomized way. We assume that data points take values in the instance space 
$\Zw$, the dataset in $\Sw = \Zw^n$, and the output is an element of the hypothesis class $\Ww$ (all assumed to be 
measurable spaces). We study the performance of the learning algorithm in terms of a loss function 
$\ell:\Ww\times\Zw\ra \real_+$. Two key objects of interest are the \emph{training error} $L(W_n,S_n) = 
\frac{1}{n}\sum_{i=1}^n \ell(w,z)$ and the \emph{test error} $\EE{\ell(w,Z')}$ of a hypothesis $w\in\Ww$, 
where the random element $Z'$ has the same distribution as the $Z_i$ and is independent of $S_n$.
The \emph{generalization error} of the algorithm is defined as\footnote{Usually the 
generalization error is defined with the opposite sign; we have made this unusual choice because it harmonizes better 
with our 
analysis technique.}
\[
 \gen(W_n,S_n) = L(W_n,S_n) - \EEcc{\ell(W_n,Z')}{W_n}~.
\]

Bounding the generalization error is one of the fundamental problems of statistical learning theory. 
Our starting point for this work is the so-called ``information-theoretic'' generalization bound proposed in the 
influential works of \citet{RZ16,RZ19} and \citet{XR17}, showing that the expected generalization error of any 
algorithm can be  bounded in terms of the mutual information between the input $S_n$ and the output $W_n=\alg(S_n)$. 
Supposing that the loss $\ell(w,Z)$ of any fixed hypothesis $w$ is $\sigma$-subgaussian, the bound takes the following 
form:
\begin{equation}\label{eq:ITbound}
 \babs{\EE{\gen(W_n,S_n)}} \le \sqrt{\frac{\sigma^2 I(W_n;S_n)}{n}}.
\end{equation}
In plain words, this guarantee expresses the intuitive property that algorithms that leak little information about the 
training data into their output generalize well. This interpretation hinges on understanding the mutual 
information as a measure of dependence between the random variables $S_n$ and $W_n$.\looseness-1

In the present work, we set out to explore other possible choices of dependence measures beyond the classic notion of 
Shannon's mutual information in the above bound. In particular, we model dependence measures as convex functions of the 
joint distribution of $W_n$ and $S_n$ and show that any such function $H$ satisfying a certain strong convexity 
property certifies a generalization bound of the form\looseness-1
\[
 \babs{\EE{\gen(W_n,S_n)}} \le \sqrt{\frac{C_{\ell,\mu,H} H\pa{P_{W_n,S_n}}}{n}},
\]
where $C_{\ell,\mu,H}$ is a constant depending on the loss function $\ell$, the data distribution $\mu$, and the 
strong-convexity properties of $H$. Specifically, this constant captures the regularity of the loss function as 
measured by a certain norm influenced by the choice of $H$. To illustrate the effectiveness of our technique, we 
provide several applications of our main result that allow us to do away with the subgaussianity assumption made in 
previous works. Some of the highlights are the following:
\begin{itemize}
 \item A generalization bound for $H$ chosen as the input-output mutual information that depends on the 
second moment of $\sup_{w\in\Ww} |\ell(w,Z) - \EEZ{\ell(w,Z)}|$ instead of the subgaussianity constant $\sigma$.
 \item A generalization bound depending on the $p$-norm distance between $P_{W_n,S_n}$ and $P_{W_n}\otimes 
P_{S_n}$ that replaces the subgaussianity constant with the $q$-th moment of the test loss $\ell(W_n,Z')$ (where $p$ 
and $q$ are positive reals satisfying $1/p + 1/q = 1$). 
 \item A generalization bound depending on the expected squared Wasserstein-2 distance between $P_{W_n|S_n}$ and 
$P_{W_n}$, and a Sobolev-type norm that replaces the subgaussianity constant.
 \item An improved generalization bound for stochastic gradient descent based on the perturbation analysis of 
\citet{NDHR21} that allows the perturbation magnitude to remain constant with $n$.
\end{itemize}

We are not the first to propose amendments to the standard bound of Equation~\eqref{eq:ITbound}. One immediate concern 
about this bound is that the mutual information may be extremely large (and even infinite) when the algorithm leaks too 
much information of the data into the output. This issue is addressed by the work of \citet{BZV20} who replaced 
$I(W_n;S_n)$ with a ``single-letter'' mutual information $I(W_n;Z_i)$ between the output and a single data point. An 
orthogonal improvement has been made by \citet{SZ20} who have introduced the idea of first conditioning on a set of 
$2n$ data points (including the training data) and measuring the generalization ability of learning algorithms by the 
mutual information between the output and the \emph{identity} of the training data points. This quantity is always 
bounded and the resulting bounds are flexible enough to recover classic generalization bounds from earlier literature, 
as shown by \citet{HDMR21}. \citet{HD20a,HD20b} provide a variety of improvements over the standard bound, such as 
proving subgaussian high-probability bounds in terms of a ``disintegrated'' version of the mutual information, and 
highlighting connections with PAC-Bayes bounds. Among other contributions, \citet{EGI21} provided generalization bounds 
in terms of R\'enyi's $\alpha$-divergences and Csisz\'ar's $f$-divergences, focusing on high-probability guarantees 
with 
subgaussian tails. 
Going beyond subgaussian losses, \citet{ZLT18} and \citet{WDSC19} provided bounds in terms of the Wasserstein
distance between $P_{W_n|S_n}$ and $P_{W_n}$ under the condition that the loss function is Lipschitz. These results 
were strengthened in multiple ways by \citet{RBTS21}, most notably by proving a ``single-letter'' variant that allowed 
them to recover several of the above-mentioned results in a unified framework. Several further improvements were made 
by \citet{NHDKR19}, \citet{HNKRD20}, who also provided applications of their bounds to study the generalization error 
of noisy iterative algorithms.

Most of these works are based on information-theoretic tools such as variational characterizations of divergences and 
direct manipulations of the resulting expressions. Our work complements this view by taking the perspective of convex 
analysis and establishing a connection between strong convexity of the dependence measure and the rate of decay of 
the generalization error. In particular, this technique allows us to establish clear conditions on the dependence 
measure under which the generalization error decays as $n^{-1/2}$. 

Our analysis is entirely based on elementary arguments from convex analysis, as covered by any introductory text 
on this subject (our personal recommendation being the excellent books of \citealp{HL01} and \citealp{Zal02}).
The key idea is bounding the generalization error via the Fenchel--Young inequality applied to the Legendre--Fenchel 
conjugate of the dependence measure $H$. We regard this conjugate as a potential function and track its changes as 
a function of the number of data points $n$ that the algorithm processes. This approach draws heavily on the 
convex-analytic analyses of online learning algorithms like Follow-the-Regularized-Leader and Mirror Descent (see, 
e.g., \citealp{Ora19,Haz16,SS12}). On an even higher level, our main idea of analyzing the performance of learning 
algorithms via a virtual online learning method is inspired by the work of \citet{ZL19}, who applied a similar idea to 
analyze the performance of Thompson-sampling-like algorithms for bandit problems. Our setup is simpler than theirs in 
that we don't have to deal with partial feedback, yet it is somewhat more abstract due to the absence of a clear sequential structure of the problem formulation we consider.

\section{Preliminaries}
Consider the setup and notation laid out in the introduction.
Our main results concern bounding the expected generalization error via tools from convex analysis, and in 
particular we will work with convex functions of joint distributions over $\Ww\times\Sw$. We 
denote the set of all probability distributions over a given set $\mathcal{H}$ as $\mathcal{P}(\mathcal{H})$ and 
the dual set of bounded functions from $\mathcal{H}$ to the reals as $\F(\mathcal{H})$. To simplify some of our 
notation 
below, we also use the shorthand notation $\Delta = \mathcal{P}(\Ww\times\Sw)$ and $\Gamma = \Pw(\Ww)$.
We denote the joint distribution of $(W_n,S_n)$ by $P_n = P_{W_n,S_n}$, the marginal distribution of $W_n$ by 
$P_{W_n}$, and use $P_0 = P_{W_n}\otimes\mu^n$ to refer to the product of the marginal distributions. We also define 
the probability kernel $\kappa(\cdot,s) = \PP{\alg(s)\in\cdot}$ for $s\in \Sw$, corresponding to the distribution of 
the output of the algorithm $\alg$.
Furthermore, for any $P\in\Delta$, we use the notation $P_{|s}\in\Gamma$ to denote the (regular 
version\footnote{We will only work with distributions for which the regular versions are well 
defined.} of the) conditional distribution of $W$ given $S_n=s$, and notice that it is a linear function of $P$. 
Indeed, for any $\lambda \in [0,1]$ and $P,P'\in\Delta$, the mixture distribution clearly satisfies $\pa{\lambda P + 
(1-\lambda) P'}_{|s} = \lambda P_{|s} + (1-\lambda) P'_{|s}$ due to the $\Sw$-marginals being fixed.
For any function $g\in\F(\Ww)$, we use the following notation to denote its expectation under a distribution 
$Q\in\Gamma$:
\[
 \iprod{Q}{g} = \EEs{g(W)}{W\sim Q}~.
\]
We sometimes refer to this bilinear map as the dual pairing between the space of bounded functions and 
probability distributions in $\Gamma$.
Furthermore, with some abuse of notation, we also define a dual pairing of joint distributions $P\in\Delta$ and 
functions $f\in\F(\Ww\times\Sw)$ as 
\[
 \iprod{P}{f} = \EEs{f(W,S)}{(W,S)\sim P} = \EES{\iprod{P_{|S}}{f(\cdot,S)}},
\]
where we have also introduced the notation $\EES{\cdot}$ to denote expectation with respect to the random dataset 
$S\sim\mu^n$. We note that all bilinear functions on $(\Gamma,\F(\Ww))$ and $(\Delta,\F(\Ww\times\Sw))$ can be represented using these dual pairings due to the Kantorovich representation theorem (cf.~Section~19.3 in \citealp{RF88}).
To see the usefulness of this notation, we define the \emph{centered loss} $\bloss(w,z) = \loss(w,z) - \EE{\loss(w,Z)}$, the $i$-th sample loss as the function $\bloss_i(s,w) = \bloss(z_i,w)$ and the $i$-th partial average loss as 
$\bL_i = \frac 1n \sum_{j=1}^i \bloss_j$, and  write the expected 
generalization error as
\[
\EE{\gen(W_n,S_n)} = \EE{\bL_n(W_n,S_n)} = \EEs{\bL_n(W,S)}{(W,S)\sim P_n} = \iprod{P_n}{\bL_n}.
\]

We aim to provide bounds on the generalization error in terms of a \emph{dependence measure} capturing the dependence 
between $W_n$ and $S_n$ as described by their joint distribution $P_n$. Technically, the dependence measure 
is a mapping from joint distributions in $\Delta$ to positive reals, that is $H:\Delta\ra\real_+$. We also 
define the \emph{conditional dependence measure} $\bH$ acting on distributions in $\Gamma$ similarly as 
$\bH:\Gamma\ra\real_+$. We assume that $\bH$ is convex and lower semicontinuous on $\Gamma$, and satisfies $\bH(P_{W_n}) = 0$. We will exclusively consider dependence measures constructed using conditional dependence measures as $H(P) = \EES{\bH(P_{|S})}$. Note that $H$ is convex in its argument $P\in\Delta$ due to $P_{|s}$ being linear in $P$ for all $s$, and is also lower semicontinuous by construction.

We recall that convexity of $\bH$ is meant in the classical sense that for all $Q,Q'\in\Gamma$ and all $\lambda\in[0,1]$, we have $\bH(\lambda 
Q + (1-\lambda)Q') \le \lambda \bH(Q) + (1-\lambda) \bH(Q')$. This 
implies that for any $Q,Q'\in\Gamma$, there exists a function $g\in\F(\Ww)$ such that
\[
 \bH(Q) \ge \bH(Q') + \iprod{Q - Q'}{g}
\]
holds. The set of all functions $g$ satisfying 
this property is called the subdifferential of $\bH$ at $Q'$ and is 
denoted by $\partial \bH(Q')$. Elements of the subdifferential are called subgradients. 
Furthermore, we say that a conditional dependence measure $\bH$ is $\alpha$-strongly convex with respect to a norm 
$\norm{\cdot}$  if the following inequality is additionally satisfied for any $g\in \partial \bH(Q')$:
\begin{equation}\label{eq:strong_convexity}
 \bH(Q) \ge \bH(Q') + \iprod{Q - Q'}{g} + \frac{\alpha}{2} \norm{Q - Q'}^2,
\end{equation}
where $\norm{\cdot}$ is some norm on the space of finite signed measures over $\Ww$. For any norm $\norm{\cdot}$, we define the associated dual norm as
\[
 \norm{f}_* = \sup_{Q,Q'\in \Gamma: \norm{Q-Q'} \le 1} \iprod{Q-Q'}{f}
\]
for any bounded function $f\in\F(\Ww)$, where the supremum is taken over all finite signed measures.

\section{Main result and proof}
We now state our main result: an upper bound on the expected generalization error in terms of the dependence measure 
$H$ 
and the dual norm $\norm{\cdot}_*$ of the loss function.
\begin{theorem}\label{thm:main}
Let $\bH$ be $\alpha$-strongly convex with respect to the norm $\norm{\cdot}$. Then, the expected generalization error of 
$\alg$ is bounded as
\[
 \babs{\EE{\gen(W_n,S_n)}} \le \sqrt{\frac{4 H(P_n) \EE{\norm{\bloss(\cdot,Z)}^2_*}}{\alpha n}}.
\]
\end{theorem}
Our proof strategy is based on a potential-based argument that draws heavily on convex-analytic tools. In particular, we 
define a potential that maps functions $f:\Ww\times\Sw\ra\real$ to reals as
\begin{equation}\label{eq:potential}
 \Phi(f) = \sup_{P\in\Delta_n} \ev{\iprod{P}{f} - H(P)},
\end{equation}
where $\Delta_n \subseteq \Delta$ is a convex set whose concrete definition will be given shortly.
We often refer to the above functional as the \emph{overfitting potential}. In words, the overfitting potential is 
the Legendre--Fenchel conjugate of the dependence measure $H$ on the set $\Delta_n$, a relationship that we 
sometimes denote as $\Phi = H^*$.
The choice of $\Delta_n$ is rather intricate and is of key importance for our proof. In order to give its precise definition, we 
first define a set of joint distributions $\ev{P_i}_{i=1}^n$ as follows: besides the already defined training 
set $S_n = \ev{Z_i}_{i=1}^n$, we define the independent ``ghost data set'' $S'_n=\ev{Z_i'}_{i=1}^n$ consisting of 
i.i.d.~samples from 
the distribution $\mu$. For each $i\in[n]$, we also define the ``mixed bag'' data set $S^{(i)}_n = 
\ev{Z_1,Z_2,\dots,Z_{i},Z'_{i+1},\dots,Z'_{n}}$.
Finally, for all $i$, we define $W^{(i)} = \alg\bpa{S_n^{(i)}}$, that is, the 
output of the learning algorithm on the $i$-th mixed bag, and define $P_i$ as the joint distribution of 
$(W^{(i)},S_n)$. 
Note that $S_n^{(0)} = S'_n$ and $S_n^{(n)} = S_n$, which explains our previously defined notation $P_n 
= P_{W_n,S_n}$ and $P_0 = P_{W_n} \otimes \mu^n$. Also notice that, by construction, all distributions $P_i$ have the 
fixed $\Sw$-marginal of $\mu^n$. Finally, for each $i$, we define $\Delta_i$ as the convex hull 
of all distributions $\ev{P_k}_{k=0}^i$: $
\Delta_i = \ev{P\in\Delta: \sum_{k=0}^i \alpha_k P_k,\, \alpha_k \ge 
0 \ \ (\forall k),\,\sum_{k=0}^i \alpha_k = 1}$.

The first step of our analysis is to pick any $\eta \in \real$ and apply the Fenchel--Young inequality to bound the generalization error as follows:
\begin{equation}\label{eq:generalization_vs_potential}
 \eta \EE{\gen(W_n,S_n)} = \eta \iprod{P_n}{\bL_n} \le H(P_n) + \Phi(\eta \bL_n).
\end{equation}
Indeed, this is easy to verify by evaluating the overfitting potential~\eqref{eq:potential} at $f = \eta \bL_n$ 
and observing that 
\[
 \Phi(\eta \bL) = \sup_{P\in\Delta_n} \ev{\eta \iprod{P}{\bL_n} - H(P)} \ge \eta \iprod{P_n}{\bL_n} - H(P_n)~.
\]
The main challenge is then to show that the overfitting potential $\Phi(\eta \bL_n)$ is of the order $\eta^2/n$ under the conditions of the theorem.

Before we can show this, it is useful to establish some basic properties of the potential $\Phi$.
We first note that $\Phi$ is convex in $f$ due to being a supremum of affine functions.
Whenever $\Phi(f)$ is bounded, it has a nonempty subdifferential $\partial \Phi(f)$ consisting of the convex hull of 
the maximizers of $\ev{\iprod{P}{f} - H(P)}$:
\[
 \partial \Phi(f) = \mbox{conv}\pa{\argmax_{P\in\Delta_n} \ev{\iprod{P}{f} - H(P)}}.
\]
Indeed, for any $P$ in the above set, the following clearly holds for any $g\in\F(\Ww\times\Sw)$:
\[
 \Phi(g) \ge \Phi(f) + \iprod{P}{g - f}.
\]
Thus, we can define the corresponding \emph{generalized Bregman divergence} as
\[
 \DDPhi{g}{f} = \Phi(g) - \Phi(f) + \sup_{P \in \partial \Phi(f)} \iprod{P}{f - g}, 
\]
where the supremum is introduced to resolve the ambiguity of the subdifferential. Notice that this is a convex 
function of $g$, being a sum of a convex function and a supremum of affine functions, and that $\DDPhi{g}{f}\ge 0$ for all $f$ and $g$ due to convexity of $\Phi$.

We are now ready to prove the following key result:
\begin{theorem}\label{thm:potential_growth}
 For any $\eta\in\real$, the overfitting potential satisfies
 \[
  \Phi(\eta \bL_n) \le \sum_{i=1}^n \DDPhi{\eta \bL_i}{\eta \bL_{i-1}}.
 \]
 \end{theorem}
\begin{proof}
We start by writing $\Phi(\eta \bL_n)$ as
\begin{equation}\label{eq:potential_telescope}
\begin{split}
 \Phi(\eta \bL_n) &= \sum_{i=1}^n \pa{\Phi(\eta \bL_i) - \Phi(\eta \bL_{i-1})} + \Phi(0)
\\
&= \sum_{i=1}^n \pa{\DDPhi{\eta \bL_i}{\eta \bL_{i-1}} - \eta \sup_{P\in\partial 
\Phi(\eta\bL_{i-1})}\iprod{P}{\bL_{i-1} - \bL_i}} 
\\
&= \sum_{i=1}^n \pa{\DDPhi{\eta\bL_i}{\eta\bL_{i-1}} + \frac{\eta}{n} \inf_{P\in\partial \Phi(\eta \bL_{i-1})} 
\iprod{P}{\bloss_i}},
\end{split}
\end{equation}
where the second line uses the definition of the generalized Bregman divergence $\mathcal{B}_\Phi$ and also that 
$\Phi(0) = 0$ due to $H$ being minimized with value zero at $P_0\in\Delta_n$, as ensured by the condition on $\bH$. It 
remains to show 
that the last term in the sum is nonpositive. 

In order to do this, we first show that for each $i$, the subdifferential of $\Phi$ includes at least one element of 
$\Delta_{i-1}$. Precisely, we show that for any $P \in \Delta_n$, there exists a $P^+\in \Delta_{i-1}$ such that 
\begin{equation}\label{eq:improvement}
\eta 
\iprod{P^+}{\bL_{i-1}} - H(P^+) \ge \eta \iprod{P}{\bL_{i-1}} - H(P),
\end{equation}
which implies that there exists a 
$P_i^*\in\Delta_{i-1}$ that achieves the maximum of $\eta \iprod{\cdot}{\bL_{i-1}} - H(\cdot)$.
To show that this is indeed the case, let us consider a fixed $P\in\Delta_n$ and write it as
$P  = \sum_{k=0}^n \alpha_k P_k$. We claim that the following choice of $P^+$ has the desired 
property~\eqref{eq:improvement}:
\[
 P^+ = \sum_{k=0}^n \alpha_k P_{k\wedge (i-1)}.
\]
To see this, we first show that $\iprod{P^+}{\bL_{i-1}} = \iprod{P}{\bL_{i-1}}$. Indeed, we note that for all $k \ge 
i$, we have 
\begin{align*}
 \iprod{P_k}{\bL_{i-1}} &= \sum_{t=1}^{i-1}\EE{\bloss(W^{(k)},Z_t)} = 
\sum_{t=1}^{i-1}\EE{\EEcc{\bloss(W^{(k)},Z_t)}{Z_{1:i-1}}} 
\\
&=  \sum_{t=1}^{i-1}\EE{\EEcc{\bloss(W^{(i-1)},Z_t)}{Z_{1:i-1}}} = 
\sum_{t=1}^{i-1}\EE{\bloss(W^{(i-1)},Z_t)} = \iprod{P_{i-1}}{\bL_{i-1}},
\end{align*}
due to the fact that the conditional distribution of $W^{(k)}|Z_{1:i-1}$ is the same as that of $W^{(i-1)}|Z_{1:i-1}$. 
This implies
\[
 \iprod{P}{\bL_{i-1}} = \frac 1n \sum_{k=0}^n \alpha_k \sum_{j=1}^{i-1}\EE{\bloss(W^{(k)},Z_j)} = 
 \frac 1n \sum_{k=0}^n \alpha_k \sum_{j=1}^{i-1}\EE{\bloss(W^{(k\wedge \pa{i-1})},Z_j)} = \iprod{P^+}{\bL_{i-1}}.
\]
It remains to show $H(P^+) \le H(P)$. To this end, let us recall the definition of the probability kernel  
$\kappa(\cdot,s)$ that characterizes the randomized output of $\alg(s)$ for any data set $s\in\Zw^n$, and recall the 
notation $S^{k} = (Z_1,\dots,Z_k,Z_{k+1}',\dots,Z_n')$. Then, we have
\[
 P_{|S} = \sum_{k=0}^n \alpha_k \EEcc{\kappa(\cdot,S^{(k)})}{S} \quad \mbox{and}\quad P^+_{|S} = \sum_{k=0}^n \alpha_k 
\EEcc{\kappa(\cdot,S^{(k\wedge\pa{i-1})})}{S}.
\]
Thus, we can write
\begin{align*}
 H(P) &= \EES{\bH(P_{|S})} = \EES{\bH\pa{\sum_{k=0}^n \alpha_k\EEcc{\kappa(\cdot,S^{(k)})}{S}}}
 \\
 &\ge \EES{\bH\pa{\sum_{k=0}^n \alpha_k\EEcc{\kappa(\cdot,S^{(k)})}{S_{1:i-1}}}} 
 \\
 &=  \EES{\bH\pa{\sum_{k=0}^n \alpha_k\EEcc{\kappa(\cdot,S^{(k\wedge\pa{i-1})})}{S}}} = H(P^+)~,
\end{align*}
where the key step follows by Jensen's inequality applied to the convex function $\bH$.
Putting the two results together proves that the inequality~\eqref{eq:improvement} indeed holds and thus $\Delta_{i-1} 
\cap \partial \Phi(\eta \bL_{i-1})$ is nonempty.

To conclude the proof, we take $P^* \in \Delta_{i-1} \cap \partial \Phi(\eta \bL_{i-1})$, write it as 
$\sum_{k=0}^{i-1} \alpha^*_k P_k$, and notice that
\begin{align*}
\inf_{P\in\partial \Phi(\eta L_{i-1})} \iprod{P}{\eta \bloss_i} &\le 
\iprod{P^*}{\eta \bloss_i} = \sum_{k=0}^{i-1} \alpha_k^* \iprod{P_k}{\eta \bloss_i}
\\
&= \sum_{k=0}^{i-1} \alpha_k^* \EE{\ell(W^{(k)},Z_i) - \ell(W^{(k)},Z_i')} = 0~,
\end{align*}
where the last step follows from observing that $W^{(k)}$ is independent of $Z_i$ by definition for $k < i$. 
Combining this inequality with Equation~\eqref{eq:potential_telescope} then proves the claim of the theorem.
\end{proof}
It remains to handle the Bregman divergences appearing in the bound of Theorem~\ref{thm:potential_growth}. The 
following lemma provides a bound that holds for strongly convex conditional dependence measures.
\begin{lemma}\label{lem:divergence_bound}
Suppose that $\bH$ is $\alpha$-strongly convex with respect to the norm $\norm{\cdot}$ whose dual norm is denoted as 
$\norm{\cdot}_*$. Then, for all $i$ and $\eta$,
\[
 \DDPhi{\eta \bL_i}{\eta \bL_{i-1}} \le \frac{\eta^2 \EEZ{\norm{\bloss(\cdot,Z)}_*^2}}{\alpha n^2}.
\]
\end{lemma}
The result follows from the well-known duality property between strong convexity and smoothness, although with 
some minor twists due to the fact that $H$ is strongly convex only in a limited sense---recall that we only we only require $\bH$ to be strongly convex on $\Gamma$, which doesn't necessarily imply strong convexity of $H$. We relegate the proof to Appendix~\ref{app:divergence_bound}.

Armed with the above results, the proof of Theorem~\ref{thm:main} is now within easy reach. By combining 
Equation~\eqref{eq:generalization_vs_potential}, Theorem~\ref{thm:potential_growth}, and 
Lemma~\ref{lem:divergence_bound}, we obtain
\begin{equation}\label{eq:almost-there}
\eta \iprod{P_n}{\bL_n} \le H(P_n) + \frac{\eta^2 \EEZ{\norm{\bloss(\cdot,Z)}_*^2}}{\alpha n}~.
\end{equation}
An upper bound can be obtained by considering $\eta >0$ and optimizing the upper bound:
\[
 \iprod{P_n}{\bL_n} \le \frac{H(P_n)}{\eta} + \frac{\eta \EEZ{\norm{\bloss(\cdot,Z)}_*^2}}{\alpha n} \le 
 \sqrt{\frac{4 H(P_n) \EEZ{\norm{\bloss(\cdot,Z)}_*^2}}{\alpha n}}~.
\]
A lower bound can be obtained by an analogous derivation for $\eta < 0$, thus concluding the proof.

\section{Applications}
\label{sec:applications}
We now instantiate our main result above to a number of specific dependence measures satisfying the condition $H(P) 
= \EES{\bH(P_{|S})}$ and $\bH(P_{W_n}) = 0$. Throughout the section, we use $Q_0$ to denote the marginal distribution 
of 
the hypotheses, consistently with our notation $P_0$ that denotes the product distribution $Q_0\otimes\mu$. We 
 often use the shorthand $P_{i|S}$ to denote $\pa{P_i}_{|S}$ for any $i$. Before providing concrete examples, we 
point out that several broadly used divergence measures satisfy the required conditions, including the entire family of 
Csisz\'ar's $f$-divergences and a family of Bregman-like divergences. Unfortunately, we could not find a 
satisfying strategy to reason about the strong convexity of these general families of dependence measures, so we 
relegate their discussion to Appendix~\ref{app:more_divergences}.\looseness-1

\subsection{Mutual information}
We start discussing the fundamental dependence measure of Shannon's mutual information, already well-studied since the 
pioneering work of \citet{RZ16,RZ19,XR17}, as mentioned in the introduction. In our framework, we can obtain generalization bounds in terms of the mutual information by taking the choice
\[
h(Q) = \DDKL{Q}{Q_0} = \int_{\Ww} \log \frac{\dd Q}{\dd Q_0} \dd Q_0,
\]
that is, the relative entropy (or Kullback--Leibler divergence) between $Q$ and the marginal 
hypothesis distribution $Q_0$. This function is well known to be $1$-strongly convex with respect to the total 
variation distance $\tvnorm{Q - Q'} = \sup_{f: \infnorm{f}\le 1} \iprod{f}{Q-Q'}$, whose dual norm is the supremum 
norm $\infnorm{f} = \sup_{w\in\Ww} \abs{f(w)}$.
Furthermore, for all $P\in\Delta_n$, the associated dependence measure $H$ is easily seen to be the 
relative entropy between the joint distributions: $H(P) = \EES{\DDKL{P_{|S}}{Q_0}} =  \DDKL{P}{P_0}$ 
(cf.~Appendix \ref{app:more_divergences}). 
Applying Theorem~\ref{thm:main} gives the following generalization bound:
\begin{corollary}\label{cor:mutual}
The generalization error of any learning algorithm satisfies
\[
 \babs{\EE{\gen(W_n,S_n)}} \le \sqrt{\frac{4 \DDKL{P_n}{P_0}\EEZ{\infnorm{\bloss(\cdot,Z)}^2}}{n}}~.
\]
\end{corollary}
Notably, this bound does not require the centered losses to be uniformly bounded for all data points, and instead it 
depends on the second moment of $\infnorm{\loss(\cdot,Z) - \EE{\loss(\cdot,Z')}}$ in terms of the random data point 
$Z$. This quantity can be finite even for heavy-tailed loss distributions whose higher moments may not exist. This is 
to 
be contrasted with the result of \citet{XR17} that requires the loss function to be subgaussian for any $w$, with the 
same constant for all hypotheses---which explicitly disallows heavy-tailed losses. In general however, the two bounds 
are incomparable due to the order of quantifiers involved in the bounds; all we can say is that both quantities are 
lower bounded by $\sup_{w\in\Ww} \mathbb{E}\bigl[\pa{\bloss(w,Z)}^2\bigr]$.

We note in passing that the guarantees of \citet{XR17} can be directly recovered by observing that the bound
\[
 \Phi(\eta \bL_n) \le \sup_{P\in\Pw(\Ww\times\Sw)} \ev{\eta \iprod{P}{\bL_n} - H(P)} = \log\EE{e^{\eta \bL_n(W_n,S_n')}} \le 
\frac{\eta^2 \sigma^2}{2 n}
\]
holds whenever the losses are $\sigma$-subgaussian, and plugging the result into the bound of 
Equation~\eqref{eq:generalization_vs_potential}. Here, the first step follows from increasing the domain of $P$ in the 
definition of $\Phi$, the second from the Donsker--Varadhan duality formula for the relative entropy, and the 
last one from the subgaussian property of the loss function. This essentially amounts 
to rewriting the proof of \citet{XR17} in our notation.

\subsection{$p$-norm divergences}
From the perspective of convex analysis, the family of $p$-norm distances is a natural candidate for defining 
dependence measures. Concretely, we define the weighted $p$-norm distance between the signed measures $Q,Q'\in\Gamma$ 
and base measure $Q_0$ as the $L_p$ distance between their Radon--Nykodim derivatives with respect to $Q_0$:
\begin{equation}\label{eq:pnorm}
 \norm{Q - Q'}_{p,Q_0} = \pa{\int_{\Ww} \pa{\frac{\dd Q}{\dd Q_0} - \frac{\dd Q'}{\dd Q_0}}^p \dd Q_0}^{1/p}.
\end{equation}
The corresponding dual norm is the $L_q$-norm defined for all $f$ as
\[
 \norm{f}_{q,Q_0,*} = \pa{\int_{\Ww} f^q \dd Q_0}^{1/q},
\]
with $q > 1$ such that $1/p + 1/q = 1$. 
It is useful to note that the distance $\norm{Q - Q_0}_{p,Q_0}^p$ is the $f$-divergence 
corresponding to $\varphi(x) = (x-1)^p$, which is known under several different names such as Hellinger divergence of 
order $p$, $p$-Tsallis divergence or simply $\alpha$-divergence with $\alpha = p$ (see, e.g., \citealp{SV16,NN11}). The 
case $p=2$ is often given special attention, and the corresponding squared norm can be seen to match Pearson's 
$\chi^2$-divergence \citep{Pea00}. We denote this divergence by $\mathcal{D}_{\chi^2}$ below.

Powers of the norm defined above exhibit different strong-convexity properties depending on the value of $p$, with two distinct regimes $p \in (1,2]$ and $p>2$. 
The following corollary summarizes the results obtained in these two regimes when setting $h(Q) = \norm{Q - Q_0}_{p,Q_0}$:
\begin{corollary}
\label{cor:pnorm}
The generalization error of any learning algorithm satisfies the following bounds:
\begin{enumerate}[label=(\alph*)]
 \item For $p\in(1,2]$,\label{cor:p_le_2}
 \[
  \babs{\EE{\gen(W_n,S_n)}} \le \sqrt{\frac{4\EES{\norm{P_{|S} - Q_0}^2_{p,Q_0}} 
\EE{\norm{\bloss(\cdot,Z)}_{q,Q_0}^2}}{(p-1) n}}~.
 \]
 \item For $p\ge 2$,\label{cor:p_ge_2}
 \[
  \babs{\EE{\gen(W_n,S_n)}} \le \frac{2 p \norm{P_n - P_0}_{\mu,p,Q_0} 
\norm{\bloss}_{\mu,q,Q_0,*}}{\pa{p-1}n^{1/p}}~.
 \]
\end{enumerate}
\end{corollary}
\citet{RBTS21} derive a comparable result for the special case $p=2$, and \citet{BGLR16} and \citet{AG18} provide 
very similar results in a PAC-Bayesian context for the entire range $p>1$, although under the stronger assumptions that 
the losses are bounded or that they always have finite variance. 
Notably, our bounds in the regime $p>2$ do not require this assumption and remain meaningful when the losses are heavy 
tailed and the $q$-th moment of the random loss is bounded only for some $q < 2$. In such cases, our result implies a 
slow rate of $n^{-(1-1/q)}$ for the generalization error, which is expected when dealing with concentration of 
heavy-tailed random variables \citep{gnedenko1954limit}. In the regime $p\in(1,2]$, our 
bound interpolates between the guarantee for $p=2$ and the one presented in Corollary~\ref{cor:mutual} as $p$ 
approaches 
$1$, at least in terms of dependence on the $L_q$-norm of the loss function. In terms of dependence on the divergence 
measures, this interpolation fails as $p$ tends to $1$, as the the squared $L_p$-divergence converges to the squared 
total variation distance which is not strongly convex. Accordingly, the bound blows up in this regime and 
Corollary~\ref{cor:mutual} gives a strictly better bound. All of these guarantees require the boundedness of $\norm{P_n 
- P_0}_{p,Q_0}$, which becomes a more and more stringent condition as $p$ increases.

All of the results in Corollary~\ref{cor:pnorm} are direct consequences of Theorem~\ref{thm:main}. The case $p=2$ is 
the simplest and can be proved by picking $h(Q) = \DDchi{Q}{Q_0}$ which gives $H(P) = 
\EES{\mathcal{D}_{\chi^2}\pa{P_{|S} \middle\| Q_0}} = \mathcal{D}_{\chi^2}\pa{P \middle\|P_0}$. Being a 
squared $2$-norm, $h$ is obviously $1$-strongly convex with respect to $\norm{Q - Q_0}_{2,Q_0}$ as it 
satisfies the condition of Equation~\eqref{eq:strong_convexity} with equality. A similar argument works for the regime 
$p\in(1,2]$, where the choice $\bH(Q) = \norm{Q - Q'}_{p,Q_0}^2$ exhibits $2(p-1)$-strong convexity with respect to the 
norm $ \norm{\cdot}_{p,Q_0}$ (see, e.g., Proposition~3 in \citealp{BCL94}, that also establishes that strong convexity 
does not hold for $p > 2$). 

The case $p\ge 2$ is more complex and it requires minor adjustments to the proof of 
Theorem~\ref{thm:main}. In this range we consider the conditional dependence measure $\bH(Q) = 
\norm{Q - Q'}_{p,Q_0}^p$. 
While this function is not strongly convex, it satisfies the following weaker notion of \emph{$p$-uniform convexity}:
\[
 \bH(Q) \ge \bH(Q') + \iprod{g}{Q - Q'} + \frac{\alpha}{2} \norm{Q - Q'}^p_{p,Q_0}
\]
with $\alpha = 2$. We refer to \citet{BCL94} who attribute this result to \citet{Cla36}. 
Following the proof of Lemma~\ref{lem:seminorm-smoothness}, we can show that $\Phi$ satisfies the following 
\emph{$q$-uniform smoothness} condition:
\[
 \DDPhi{f}{f'} \le \frac{1}{\alpha^{q-1}} \norm{f - f'}_{q,Q_0,*}^q.
\]
Replacing the bound of Lemma~\ref{lem:divergence_bound} with this inequality in the proof of Theorem~\ref{thm:main}, we 
arrive to the following analogue of Equation~\eqref{eq:almost-there} which then directly implies the claimed result after optimizing $\eta$:
\[
 \eta \iprod{P_n}{\bL_n} \le H(P_n) + \frac{\eta^q \EE{\norm{\bloss(\cdot,Z)}_{q,Q_0,*}^q}}{2^{q-1} n^{q-1}}~.
\]

\subsection{Smoothed relative entropy}\label{sec:wasserstein}
Let us now suppose that $\Ww = \real^d$ and consider a smoothed version of the relative entropy, defined via the 
Gaussian smoothing operator $G_\sigma$ that acts on any distribution $Q$ as $G_\sigma Q = \int_{\Ww} \mathcal{N}(w, 
\sigma^2 I) 
\dd Q(w)$,
 where $\mathcal{N}(w,\sigma^2 I)$ is the $d$-dimensional Gaussian distribution with mean $w$ and covariance 
$\sigma^2 I$. Using this operator, we define the smoothed relative entropy as
$\DDsigma{Q}{Q'} = \DDKL{G_\sigma Q}{G_\sigma Q'}$ and set $h(Q) = \DDsigma{Q}{Q_0}$. Similarly, we define 
the smoothed total variation distance between $Q$ and $Q'$ as $
 \norm{Q - Q'}_\sigma = \norm{G_\sigma Q - G_\sigma Q'}_{\mathrm{TV}}$. Both of these divergences have the attractive 
property that they remain meaningfully bounded under much milder assumptions than their unsmoothed counterparts (e.g., 
even when the supports of $Q$ and $Q'$ are disjoint).

It is straightforward to verify that the Bregman divergence associated with $h$ satisfies 
\[
\DDh{Q}{Q'} = \DDsigma{Q}{Q'} \ge \frac 12 \norm{G_\sigma\pa{Q - Q'}}^2_{\mathrm{TV}} = \frac 12 \norm{Q - Q'}_\sigma^2,
\]
thus implying $1$-strong convexity in terms of the smoothed total variation distance. The dual norm of the smoothed TV 
distance is defined as $\norm{f}_{\sigma,*} = \sup_{\norm{Q-Q'}_\sigma \le 1} \iprod{f}{Q-Q'}$, which, together with the above arguments, immediately implies the following result:
\begin{corollary}\label{cor:smoothed_plain}
For any $\sigma>0$, the generalization error of any learning algorithm satisfies
 \[
\babs{\EE{\gen(W_n,S_n)}} \le \sqrt{\frac{4 \EES{\DDsigma{P_{n|S}}{Q_0}} 
\EEZ{\norm{\bloss(\cdot,Z)}^2_{\sigma,*}}}{n}}. 
\]
\end{corollary}
A useful fact is that the smoothed relative entropy can be upper-bounded in terms of the squared Wasserstein-2 distance 
as $\DDsigma{Q}{Q'} \le \frac{1}{2\sigma^2} \mathbb{W}_2^2(Q,Q')$. For completeness, we give the precise definition of the Wasserstein distance $\mathbb{W}_2$ and a direct proof of this result in Appendix~\ref{app:wasserstein}.
%We refer to Lemma~4 in \citet{NDHR21} for a direct proof. 
It remains to be shown that the dual norm $\norm{\ell(\cdot,z)}_{\sigma,*}$ can be bounded meaningfully.
By the intuitive properties of the smoothed total variation distance, one can reasonably expect this norm to capture the smoothness properties of the loss function, and it is small whenever $\ell(\cdot,z)$ is bounded and highly smooth. In what follows, we show an upper bound on this norm that holds 
for a class infinitely smooth functions. Specifically, we say that a function $f$ is infinitely smooth if all of its higher-order directional derivatives exist and satisfy $D^j 
f(w|v_1,v_2,\dots,v_j) \le \beta_j$ for all directions $v_1,v_2,\dots,v_j$, all $w\in\Ww$, and all $j$.
For such functions, the following lemma provides an upper bound on $\norm{f}_{\sigma,*}$:
\begin{lemma}\label{lem:smoothing}
Suppose that $f$ is infinitely smooth in the above sense.
Then, the dual norm $\norm{f}_{\sigma,*}$ satisfies $\norm{f}_{\sigma,*} \le \sum_{j=0}^\infty \bpa{\sigma \sqrt{d}}^j 
\beta_j$.
\end{lemma}
The proof is based on a successive smoothing argument and is provided in Appendix~\ref{app:smoothing}.
With the help of this lemma, we may pick $\sigma = 1/(2\sqrt{d})$ and obtain the following result:
\begin{corollary}\label{cor:wasserstein}
Suppose that $\ell(\cdot,z)$ is infinitely smooth for all $z$ with $\beta_j \le \beta$ for all $j\ge 0$. Then, the 
generalization error of any learning algorithm satisfies
 \[
\babs{\EE{\gen(W_n,S_n)}} \le \sqrt{\frac{8 \beta d \EES{\mathbb{W}_2^2\pa{P_{n|S},Q_0}}}{n}}.
\]
\end{corollary}
We are not aware of any directly comparable results in the literature. \citet{ZLT18}, \citet{WDSC19} and \citet{RBTS21} 
provide vaguely similar guarantees that depend on the Wasserstein-1 distance and only require bounded first 
derivatives, but it is not clear if these bounds are decreasing with the sample size $n$ in general. Whenever all hypotheses satisfy $\twonorm{w}\le R$ for some $R$, the result stated above implies an upper bound on the expected generalization error that scales as $R\sqrt{\beta d / n}$ whenever all hypotheses satisfy $\twonorm{w}\le R$ for some $R$, which is directly comparable with what one might obtain via a straightforward uniform convergence argument involving the covering number of Lipschitz functions on a bounded domain (see, e.g., \citealp{Dud84}). The dependence on the dimension $d$ of such guarantees can be relaxed or completely removed when assuming more structure about the loss function \citep{Bar98,WSS00,Zha02}. Whether such arguments can be applied to remove the dependence on $d$ from the above bound is a curious problem we leave open for future research.

Finally, we further specialize our bound above to derive an upper bound on the generalization error of 
stochastic gradient descent, building on the results of \citet{NDHR21}. In particular, their Theorem~5 provides 
an upper bound on the divergence $\EES{\DDsigma{P_{n|S}}{Q_0}}$ for this algorithm.
Applying this result and borrowing all notation from said paper, we state the following bound:
\begin{corollary}
Suppose that $\ell(\cdot,z)$ is infinitely smooth for all $z$ with $\beta_j \le \beta$ for all $j\ge 0$. 
Furthermore, suppose that the variance of the gradients is uniformly upper bounded by $v$ for all $w$. Then, for any 
$\sigma \le d^{-1/2}/2$, the generalization error of the final iterate produced by single-pass SGD with stepsize sequence $(\eta_t)_t$ satisfies
\[
 \babs{\EE{\gen(W_T,S_n)}} = \mathcal{O}\pa{\sqrt{\beta \sum_{t=1}^n \eta_t^2 \pa{\frac{v}{\sigma^2} + \beta^2 d}}}.
\]
In particular, choosing $\sigma = d^{-1/2}/2$ and $\eta_t = 1/n$, the generalization error decays as
$\mathcal{O}\bpa{\sqrt{d/n}}$.
\end{corollary}
The major advantage of the bounds we have just obtained is that they allow deriving nontrivial guarantees while 
keeping $\sigma$ constant. This is to be contrasted with the results of \citet{NDHR21}, whose technique required 
$\sigma$ to approach zero as $n$ increases. The price we had to pay for this result is assuming that the loss function 
is differentiable infinitely many times, as opposed to being differentiable only once as required by their previous result.

\section{Conclusion}
We discuss some implications and potential directions for future work below.

\paragraph{High-probability bounds.}
The most interesting open question we leave behind is whether or not our techniques can be extended to 
provide high-probability guarantees. This seems like a serious challenge in light of the lower bounds of 
\citet{BMNSY18} who show that low mutual information is not sufficient to obtain subgaussian concentration bounds on 
the excess risk (Proposition~11). More broadly, it suggests that the strong convexity condition we identify in our work may be insufficient for achieving such strong results.  It remains to be seen if it is possible to express further 
conditions on the dependence measure in the language of convex analysis to overcome this burden.

\paragraph{Other dependence measures.}
The few examples we provided in Section~\ref{sec:applications} admittedly only serve to illustrate our main result, and 
it is quite possible that several stronger guarantees can be derived using our techniques. We are particularly curious if 
strong convexity of the Wasserstein distances could be directly demonstrated and our Corollary~\ref{cor:wasserstein} 
could be proved in a less roundabout way. On the same note, we are equally interested in improving the bound of 
Lemma~\ref{lem:smoothing} on the dual norm of the smoothed total variation distance, particularly in terms of removing 
the condition on the infinite differentiability of $f$ and improving the dependence on the dimension $d$. We conjecture that these should both be possible by a more careful analysis that exploits the properties of Gaussian smoothing more effectively.\looseness-1

\paragraph{Single-letter guarantees.}
We mention without proof that it is possible to prove the following ``single-letter'' version of our 
main result:
\[
 \babs{\EE{\gen(W_n,S_n)}} \le \frac{1}{n} \sum_{i=1}^n \sqrt{\frac{\EE{\bH(P_{n|Z_i})} 
\EE{\norm{\bloss(\cdot,Z)}^2_*}}{\alpha}}.
\]
This can be achieved by choosing $H(P) = \frac 1n \sum_{i=1}^n \EE{\bH(P_{|Z_i})}$ instead of $H(P) = \EE{\bH(P_{|S})}$ in the 
definition of the overfitting potential. One can verify that all steps in the proof of 
Theorem~\ref{thm:potential_growth} continue to work for this choice, and the bound of Lemma~\ref{lem:divergence_bound} 
can also be shown to hold for an appropriately adjusted version of the lifted dual norm $\norm{\cdot}_{\mu,*}$. As 
shown by \citet{BZV20}, this version can sometimes result in improved upper bounds, but we also remark that this is 
only possible for divergences that satisfy $\sum_{i=1}^n \bH(P_{n|Z_i}) \le \bH(P_{n|S})$ which holds for the mutual 
information with equality due to the chain rule. \looseness-1

\paragraph{Connection with online learning.} The proof of our main result is based on convex-analytic tools that are common in the analysis of online learning algorithms, and particularly Follow-the-Regularized-Leader (FTRL) methods (cf.~Chapter~7 of \citealp{Ora19}). While we have presented our proof in a self-contained manner, it is possible to take an alternative route and prove our main theorem using a more general reduction to regret minimization, by connecting the generalization error with the regret of a ``virtual online learning'' algorithm run in an appropriately designed sequential game. The construction goes as follows: Consider a sequence of rounds $t=1,2,\dots,n$, where in each round $t$, the online learner picks a joint distribution $\wt{P}_t \in \Delta_n$ and gains a reward $\biprod{\wt{P}_t}{\bloss_t}$. The regret of this online learner against the comparator $P_n$ is defined as $\text{Regret}_n(P_n) = \sum_{t=1}^n \biprod{P_n - \wt{P}_t}{\bloss_t}$, so that the generalization error can be written as 
\[
\EE{\gen(W_n,S_n)} = \iprod{P_n}{\bL_n} = \sum_{t=1}^n \biprod{\wt{P}_t}{\bloss_t} + \text{Regret}_n(P_n).   
\]
We then proceed by considering an FTRL algorithm with regularizer $H$, whose updates are calculated as  $\wt{P}_t = \argmax_{P\in\Delta_n} \ev{\eta \iprod{P}{\bL_{t-1}} - H(P)}$, and follow the ideas from the classical FTRL analysis to show an upper bound on the regret of this method. 
%More concretely, the potential decomposition of Equation~\eqref{eq:potential_telescope} and Lemma~\ref{lem:divergence_bound} are based on standard arguments. 
The main technical challenge specific to our setting is showing that the predictions of FTRL satisfy $\biprod{\wt{P}_t}{\bloss_t} = 0$ for all $t$, which makes up the bulk of the proof of our Theorem~\ref{thm:potential_growth}. We are confident that this latter argument can be adapted to other algorithms beyond FTRL. We finally note that, after the first publication of this work, we have discovered several connections with the works of \citet{Zha02} and \citet{KST08}, who provided reductions from online learning to bounding complexity measures of function classes. We believe that combining their techniques with the tools developed in our work can lead to some exciting future progress.\looseness-1

\paragraph{Faster rates.} Another curious question is if our techniques can be extended to provide rates that decay faster than $\OO(1/\sqrt{n})$. We believe that this should indeed be possible via a more sophisticated analysis technique. A potential approach leading to faster rates could be to take advantage of the fact that the loss sequence in our online learning construction is far from being adversarial, which can allow proving regret bounds that are potentially much better than the worst-case bound that our current analysis is based on. We refer to \citet{vEGMRW15} for an overview of the type of regularities that one can exploit in online learning in order to get such faster rates. Another possibility would be to consider bounding a strongly convex proxy to the generalization error, which may allow proving faster rates using further tools from online learning theory. More concretely, the classic PAC-Bayesian bounds of \citet{LS01} and \citet{See02} can be thought of bounding such a strongly convex proxy, and their results can indeed lead to rates of order $1/n$ in specific cases---see Sections~3.2.3 and~3.2.4 in \citet{Alq21} for a modern framing of these results.

\acks{G.~Lugosi was supported by by the Spanish Ministry of Economy and Competitiveness, Grant PGC2018-101643-B-I00 and FEDER, EU. G.~Neu was supported by the European Research Council (ERC) under the European Union’s Horizon 2020 research and innovation programme (Grant agreement No.~950180). The authors wish to thank the four anonymous reviewers for their helpful feedback, and also Csaba Szepesv\'ari, Peter Bartlett, Peter Gr\"unwald and Borja Rodriguez G\'alvez for insightful discussions that helped shape some of the results presented in this paper. We finally thank Wojciech Kot\l owski, Ohad Shamir, Roi Livni, Matus Telgarsky, and Adam Block for further useful comments and pointers to relevant literature that we have missed earlier.}

\bibliographystyle{abbrvnat}
\bibliography{ngbib,gengen,shortconfs}

\appendix

\section{Omitted proofs}
\subsection{The proof of Lemma~\ref{lem:divergence_bound}}\label{app:divergence_bound}
We first observe that whenever $\bH$ is $\alpha$-strongly convex with respect to $\norm{\cdot}$, then $H$ is also 
strongly convex on 
$\Delta_n$ with respect to the ``lifted'' norm $\norm{P - P'}_\mu = \bpa{\mathbb{E}_S\bigl[\bigl\|P_{|S} - 
P_{|S}'\bigr\|^2\bigr]}^{1/2}$. Indeed, this follows 
from the following simple calculation:
\begin{align*}
 H(\lambda P + (1-\lambda)P') &= \mathbb{E}_S\bigl[\bH\bpa{\pa{\lambda P + (1-\lambda)P'}_{|S}}\bigr] = 
\EESb{\bH\bpa{\lambda P_{|S} + (1-\lambda)P'_{|S}}}
\\
&\le \EESB{\lambda \bH\bpa{P_{|S}} + (1-\lambda) \bH\bpa{P'_{|S}} - \frac{\alpha}{2} \bnorm{P_{|S} - P_{|S}'}^2}
\\
&= \lambda H\pa{P} + (1-\lambda) H\pa{P'} - \frac{\alpha}{2} \EESB{\bnorm{P_{|S} - P_{|S}'}^2}.
\end{align*}
Here, the first step uses the definition of $H$, the second the affinity of the conditional distributions in the joint 
distributions, the third step the strong convexity of $\bH$, and the last one uses the definition of $H$ one more time. 

Here we pause to point out that $\Delta_n$ is supported on an affine subspace of $\Delta$, and that $\norm{\cdot}_\mu$ 
only acts as a norm on the subspace of signed measures in $\Delta_n - \Delta_n$. Note that the dual of this 
Banach space is broader than the set of functions $\F(\Ww,\Sw)$ integrable under all joint distributions in $\Delta$, 
as only integrability with respect to measures in $\Delta_n$ is required. For this reason, we cannot appeal to the 
traditional duality results between strong convexity and strong smoothness as these require reasoning about the dual 
norm of $\norm{\cdot}_\mu$. Nevertheless, we can still obtain the same results via the notion of \emph{dual 
seminorm}, defined for all $f\in\F(\Ww,\Sw)$ as $\norm{f}_{\mu,*} = \pa{\EES{\norm{f(\cdot,S)}_*}}^{1/2}$. 
The dual seminorm satisfies all properties of a norm except positive 
definiteness, as it may be zero even when $f$ is not identically zero (albeit only on a set with $\mu$-measure zero). 
Most importantly for our analysis, it also satisfies the following property for all $P,P'\in\Delta_n$ and all 
$f\in\F(\Ww,\Sw)$:
\begin{align*}
 \iprod{P-P'}{f} &= \EESB{\biprod{P_{|S} - P'_{|S}}{f(\cdot,S)}} \le \EESB{\bnorm{P_{|S} - 
P'_{|S}}\norm{f(\cdot,S)}_{*}}
 \\
 &\le \pa{\EESB{\bnorm{P_{|S} - P'_{|S}}^2}}^{1/2}\cdot \pa{\EES{\norm{f(\cdot,S)}_{*}^2}}^{1/2} 
 = \norm{P - P'}_{\mu}\norm{f}_{\mu,*}~.
%= \norm{P - P'}_\mu  \norm{f}_{\mu,*}
\end{align*}
Here, we have used the definition of the norm $\norm{\cdot}$ and the dual norm $\norm{\cdot}_*$, and the 
Cauchy--Schwarz inequality. These properties are sufficient to show that the $\alpha$-strong convexity of $H$ on 
implies 
$1/\alpha$-strong smoothness of its Legendre--Fenchel conjugate $\Phi = H^*$ with respect to $\norm{\cdot}_{\mu,*}$. We 
defer the proof to Appendix~\ref{app:seminorm-smoothness}.
The proof is now concluded by applying this result as
\begin{equation}\label{eq:holder}
\begin{split}
 \DDPhi{\eta L_i}{\eta L_{i-1}} 
 &\le 
 \frac{\norm{\eta \pa{\bL_i - \eta \bL_{i-1}}}_{\mu,*}^2}{\alpha} = 
\frac{\eta^2\EES{\norm{\bloss_i(\cdot,S)}_*^2}}{\alpha n^2} = \frac{\eta^2\EES{\norm{\bloss(\cdot,Z_i)}_*^2}}{\alpha 
n^2}~,
\end{split}
\end{equation}
where the equalities are direct consequences of the definitions.
\qed

\subsection{Strong-convexity / smoothness duality}\label{app:seminorm-smoothness}
\begin{lemma}\label{lem:seminorm-smoothness}
Let $f$ and $f'$ be two integrable functions under all distributions in $\Delta_n$ and let $P \in \partial \Phi(f)$ 
and $P' \in \partial \Phi(f')$. Suppose that $\bH$ is $\alpha$-strongly convex with respect to $\norm{\cdot}$ and thus 
$H$ is $\alpha$-strongly convex on $\Delta_n$ with respect to $\norm{\cdot}_\mu$. Then, $\Phi = H^*$ satisfies
 \[
  \DDPhi{f}{g} \le \frac{1}{\alpha} \norm{f - f'}_{\mu,*}^2.
 \]
\end{lemma}
\begin{proof}
 Let $s_P \in \partial H(P)$ and $s_{P'} \in \partial H(P')$. Then, by first-order optimality of $P$ and $P'$, we have
 \begin{align*}
  \iprod{s_P - f}{P - P'} &\le 0\\
  \iprod{s_{P'} - f'}{P' - P} &\le 0.
 \end{align*}
 Summing the two inequalities, we get
 \[
  \iprod{s_{P'} - s_{P}}{P - P'} \le \iprod{P' - P}{f' - f}.
 \]
Now, using the strong convexity of $H$, we get
\begin{align*}
 H(P) \ge H(P') + \iprod{s_{P'}}{P-P'} + \frac{\alpha}{2}\norm{P-P'}_{\mu}^2\\
 H(P') \ge H(P) + \iprod{s_P}{P'-P} + \frac{\alpha}{2}\norm{P-P'}_{\mu}^2.
\end{align*}
Summing these two inequalities then gives
\[
 \alpha \norm{P-P'}_{\mu}^2 \le \iprod{s_P - s_{P'}}{P - P'}.
\]

Combining both inequalities above, we obtain
\[
 \alpha \norm{P'-P}_{\mu}^2 \le \iprod{P - P'}{f-f'} \le \norm{P-P'}_{\mu} \norm{f-f'}_{\mu,*} ,
\]
where we crucially used a key property of $\norm{\cdot}_{\mu,*}$ established in Equation~\eqref{eq:holder}.
This yields
\begin{equation}\label{eq:subgrad_norm}
 \norm{P-P'}_{\mu} \le \frac{1}{\alpha} \norm{f-f'}_{\mu,*}.
\end{equation}
Now, by the mean value theorem, there exists an $f_\lambda = \lambda f + (1-\lambda) f'$ with $\lambda\in[0,1]$ such 
that $P_\lambda \in \partial \Phi(f_\lambda)$ and
\begin{align*}
 \Phi(f) &= \Phi(f') + \iprod{P_{\lambda}}{f-f'}
 \\
 &= \Phi(f') + \iprod{P'}{f-f'} + \iprod{P_{\lambda} - P'}{f-f'}
 \\
 &\le \Phi(f') + \iprod{P'}{f-f'} + \norm{P_{\lambda} - P'}_{\mu}\norm{f-f'}_{\mu,*}
 \\
 &\qquad\qquad\mbox{(by Equation~\eqref{eq:holder})}
 \\
 &\le \Phi(f') + \iprod{P'}{f-f'} + \frac{1}{\alpha} \norm{f_\lambda - f'}_{\mu,*}\norm{f-f'}_{\mu,*}
 \\
 &\qquad\qquad\mbox{(by Equation~\eqref{eq:subgrad_norm})}
 \\
 &= \Phi(f') + \iprod{P'}{f-f'} + \frac{\lambda}{\alpha} \norm{f - f'}^2_{\mu,*}.
 \\
 &\le \Phi(f') + \iprod{P'}{f-f'} + \frac{1}{\alpha} \norm{f - f'}^2_{\mu,*}.
\end{align*}
The proof is completed by recalling that $P'\in\partial\Phi(f')$ and the definition of the Bregman divergence, and 
reordering the terms.
\end{proof}

\subsection{The proof of Lemma~\ref{lem:smoothing}}\label{app:smoothing}
For clarity, we start by formalizing the notion of directional derivatives of $f$ via the following recursive 
definition: $D^0 f = f$ and for each $j>0$, we 
define $D^j f:\Ww \times B_1^j$ as 
\[
D^j f(w|v_1,v_2,\dots,v_j) = \lim_{c\ra 0} \frac{D^{j-1} f(w + c v_j|v_1,v_2,\dots,v_{j-1}) - D^{j-1} 
f(w|v_1,v_2,\dots,v_{j-1})}{c},
\]
where $B_1$ denotes the Euclidean unit ball $B_1 = \ev{v\in\real^d: \twonorm{v} = 1}$. 
Notice that $D^j$ is linear in $f$.

The proof itself is based on the following successive smoothing argument: we begin by smoothing the original function 
$f$ using the conjugate of the smoothing operator $G_\sigma^*$, then smoothing out the residual $f - G_\sigma^* f$ and 
continue indefinitely. As we show, the residuals decay rapidly at a rate determined by the higher-order derivatives of 
the original function $f$.
To make this argument precise, 
we let $f_0 = f$ and recursively define $f_{j+1} = f_j - G_\sigma^* f_j$, so that we can write
\begin{align*}
 \iprod{Q-Q'}{f} &= \iprod{Q-Q'}{G_\sigma^* f} + \iprod{Q-Q'}{f - G_\sigma^* f} = \iprod{G_\sigma\pa{Q-Q'}}{f_0} + 
\iprod{Q-Q'}{f_1}
\\
&= \iprod{G_\sigma \pa{Q-Q'}}{f_0} + \iprod{Q-Q'}{G_\sigma^* f_1} + \iprod{Q-Q'}{f_1 - G_\sigma^* f_1}
\\
&= \iprod{G_\sigma \pa{Q-Q'}}{f_0} + \iprod{G_\sigma \pa{Q-Q'}}{f_1} + \iprod{Q-Q'}{f_2} + \dots
\\
&= \sum_{j=0}^\infty \iprod{G_\sigma \pa{Q-Q'}}{f_j} \le \norm{Q - Q'}_\sigma \sum_{j=0}^\infty \norm{f_j}_\infty,
\end{align*}
where the last step follows from  H\"older's inequality.

It remains to relate $\norm{f_j}_\infty$ to the derivatives of the original function $f$. 
To this end, 
let $\xi$ denote a Gaussian vector distributed as $\mathcal{N}(0, \sigma^2 I)$, and
note that for all $j$, we have
\begin{align*}
\infnorm{f_j} &= \sup_w \bigl|f_{j-1}(w) - \EE{f_{j-1}(w+\xi)}\bigr|
\\
&
\le \sup_w \EE{\twonorm{\xi} \cdot \abs{\frac{f_{j-1}(w) - f_{j-1}(w+\xi)}{\twonorm{\xi}}}}
\\
&
\le \EE{\twonorm{\xi}} \sup_w \sup_{v_1\in B_1}\abs{D^1 f_{j-1}(w|v_1)}
\\
&\le \pa{\sigma \sqrt{d}} \sup_w \sup_{v_1\in B_1} \abs{\EE{D^1 f_{j-2}(w|v_1) - D^1 f_{j-2}(w + \xi|v_1)}}
\\
&\le \pa{\sigma \sqrt{d}} \sup_w \sup_{v_1\in B_1} \EE{\twonorm{\xi}\cdot \abs{\frac{D^1 f_{j-2}(w|v_1) - D^1 f_{j-2}(w 
+ 
\xi|v_1)}{\twonorm{\xi}}}}
\\
&\le \pa{\sigma \sqrt{d}} \EE{\twonorm{\xi}} \sup_w \sup_{v_1,v_2\in B_1} \abs{D^2 f_{j-2}(w|v_1,v_2)}
\\
&\le \dots \le \pa{\sigma \sqrt{d}}^j \sup_w \sup_{v_1,v_2,\dots,v_j\in B_1} \abs{D^j f(w|v_1,v_2,\dots,v_j)}\le 
\beta_j~.
\end{align*}
Here, we have used the bound $\EE{\twonorm{\xi}} \le \sigma \sqrt{d}$ several times.
Putting this together with the previous bound proves the claim.
\qed

\subsection{Wasserstein distance and smoothed relative entropy}\label{app:wasserstein}
This section provides some results supporting the claims made in Section~\ref{sec:wasserstein}. We first give a precise definition for the Wasserstein distance between two distributions $Q,Q'\in\Gamma$. For the sake of concreteness, we only give the defintion for the distance metric given by the Euclidean distance on $\real^d$, and refer the reader to the book of \citet{Vil03} for a more general treatment. Letting $\Pi(Q,Q')$ denote the set of joint distributions on $\Ww\times\Ww$ with marginals $Q$ and $Q'$, the squared Wasserstein-2 distance between $Q$ and $Q'$ is defined as
\[
 \mathbb{W}_2(Q,Q') = \inf_{\pi \in \Pi(Q,Q')} \int_{\Ww\times\Ww} \twonorm{w - w'}^2 \dd \pi(w,w').
\]
The following lemma (whose proof is largely based on the proof of Lemma~4 of \citealp{NDHR21}) provides a bound on the smoothed relative entropy in terms of the squared Wasserstein-2 distance:
\begin{lemma}\label{lem:spreadbound}
% Let $X$ and $Y$ be random variables taking values in $\real^d$ with bounded second moments and let 
% $\sigma > 0$. Letting $\varepsilon\sim\N(0,\sigma^2I)$ be independent of $X$ and $Y$, the relative entropy 
% between the distributions of $X + \varepsilon$ and $Y + \varepsilon$ is bounded as
Let $W$ and $W'$ be two random variables on $\real^d$ with respective laws $Q$ and $Q'$. For any $\sigma > 0$, the smoothed relative entropy between $Q$ and $Q'$ is bounded as
\[
%  \DD{P_{X+\varepsilon}}{P_{Y+\varepsilon}} \le \frac{1}{2\sigma^2} \EE{\twonorm{X - Y}^2}.
\DDsigma{Q}{Q'} \le \frac{1}{2\sigma^2} \EE{\twonorm{W - W'}^2}.
\]
\end{lemma}
\begin{proof}
Let us consider a fixed coupling $\pi \in \Pi(Q,Q')$ and observe that the smoothed distributions $G_\sigma Q$ and $G_\sigma Q'$ can be respectively written as 
 \[
  G_\sigma Q = \int_{\Ww\times\Ww} \mathcal{N}(w,\sigma^2 I) \dd \pi(w,w') \quad\mbox{and}\quad G_\sigma Q' = 
\int_{\Ww\times\Ww} \mathcal{N}(w',\sigma^2 I) \dd \pi(w,w').
 \]
Using this observation, we can write
\begin{align*}
 \DDsigma{Q}{Q'} &= 
 \DD{\int_{\Ww\times\Ww} \mathcal{N}(w,\sigma^2 I) \dd \pi(w,w')}{\int_{\Ww\times\Ww} \mathcal{N}(w',\sigma^2 
I) \dd \pi(w,w')}
\\
&\le 
\int_{\Ww\times\Ww} \DD{\mathcal{N}(w,\sigma^2I)}{\mathcal{N}(w',\sigma^2I)} \dd \pi(w,w') 
\\
&= \frac{1}{2\sigma^2} \int_{\Ww\times\Ww} \twonorm{W-W'}^2 \dd \pi(w,w'),
\end{align*}
where the second line uses Jensen's inequality and the joint convexity of $\DD{\cdot}{\cdot}$ in its arguments, and the 
last line follows from noticing that $\DD{\mathcal{N}(x,\Sigma)}{\mathcal{N}(y,\Sigma)} = \frac 12 \norm{x 
- y}_{\Sigma^{-1}}^2$ for any $x,y$ and any symmetric positive definite covariance matrix $\Sigma$. The result then follows from taking the infimum with respect to $\pi$ on the right-hand side.
\end{proof}

\section{Further dependence measures}\label{app:more_divergences}
Besides the examples already discussed in depth in Section~\ref{sec:applications}, there are several other 
potentially interesting divergences that fit into our framework. Here we review two such classes: Csisz\'ar's 
$f$-divergences and a family of Bregman-style divergences. 
A useful tool for studying the strong-convexity properties of $h$ is its associated Bregman divergence defined for any 
$Q,Q'\in\Gamma$ as
\[
 \DDh{Q}{Q'} = h(Q) - h(Q') - \iprod{g}{Q - Q'},
\]
where $g\in\partial h(Q')$ is an arbitrary element of the subdifferential of $h$ at $Q'$. It is easy to see that the 
strong convexity of $h$ is equivalent to $\DDh{Q}{Q'} \ge \frac{\alpha}{2} \norm{Q-Q'}^2$ for all $Q,Q'$, 
independently of the choice of $g$. We will give expressions for the Bregman divergence for the above-mentioned two 
classes of divergences, and state some (rather limiting) sufficient conditions for their strong convexity. Similar 
arguments can be applied to other families of information-theoretic divergences such as R\'enyi's $\alpha$-divergences 
\citep{renyi1961measures,VEH14}.

\subsection{$f$-divergences}
Introduced by \cite{renyi1961measures} and
studied by \cite{csiszar1964informationstheoretische}, $f$-divergences are a generalization of the relative entropy and 
the $\chi^2$ divergence discussed in Section~\ref{sec:applications}.
Letting $\varphi:\real_+\ra\real$ be a convex function with $\varphi(1) = 0$, this divergence is defined\footnote{We 
use 
$\varphi$ instead of the more common $f$ to avoid clash with our notation for functions in $\F(\Ww)$ and 
$\F(\Ww\times\Sw)$.} for $Q,Q'\in\Gamma$ with $Q'\ll Q$ as
\[
  \mathcal{D}_\varphi\pa{Q\|Q'} = \int_\Ww \varphi\pa{\frac{\dd Q}{\dd Q'}} \dd Q'.
\]
Then a conditional dependence measure may be defined as $h(Q) = \mathcal{D}_\varphi\pa{Q\|Q_0}$, and its associated 
dependence measure can be simply seen to be 
\[
H(P) = \EES{\int_{\Ww} \varphi\pa{\frac{\dd P_{|S}}{\dd Q_0}} \dd Q_0} = \int_{\Ww,\Sw} \varphi\pa{\frac{\dd 
P}{\dd P_0}} \dd P_0 = \mathcal{D}_\varphi\pa{P\|P_0},
\]
where we have also extended our definition of $f$-divergences to joint distributions over $\Ww\times\Sw$ in a natural 
way. In the above calculation, we have crucially exploited the fact that for all $P\in\Delta_n$, $\frac{\dd P_{|s}}{\dd 
Q_0} = \frac{\dd P}{\dd P_0}(\cdot,s)$ holds due to the $\Sw$-marginals of all such distributions $P$ being fixed.

The resulting dependence measure is clearly convex. In order to study its strong convexity, it is insightful to 
suppose that $\varphi$ is twice differentiable with its first and second derivatives denoted by $\varphi'$ 
and $\varphi''$. A second-order Taylor expansion of the univariate function $u(\lambda) = h(\lambda Q' + 
(1-\lambda) Q)$ at zero reveals that for any $Q,Q'$, there exists a $\lambda \in [0,1]$ such that
\begin{eqnarray}
  \label{eq:f_expansion}
h(Q') & = & h(Q) + \int_\Ww \varphi'\pa{\frac{\dd Q}{\dd Q_0}} 
            \pa{\frac{\dd Q'}{\dd Q_0} - \frac{\dd Q}{\dd Q_0}} \dd Q_0
\nonumber \\ & &+ \int_\Ww \varphi''\pa{\lambda \frac{\dd Q'}{\dd Q_0} + (1-\lambda) 
\frac{\dd Q}{\dd Q_0}} \pa{\frac{\dd Q}{\dd Q_0} - \frac{\dd Q'}{\dd Q_0}}^2\dd Q_0
\nonumber \\
      &= & h(Q) + \iprod{\varphi' \circ \frac{\dd Q}{\dd Q_0}}{Q' - Q} 
           \nonumber \\ & &
                            + \int_{\Ww} \varphi''\pa{\lambda \frac{\dd Q'}{\dd Q_0} + 
(1-\lambda) \frac{\dd Q}{\dd Q_0}} \pa{\frac{\dd Q}{\dd Q_0} - \frac{\dd Q'}{\dd Q_0}}^2\dd Q_0~.
\end{eqnarray}
Since $\varphi''\ge 0$, this immediately shows that $\varphi' \circ \frac{\dd Q}{\dd Q_0} \in \partial h(Q)$. 
Furthermore, it shows that whenever $\varphi''\pa{\frac{\dd Q}{\dd Q_0}} \ge \alpha$ holds for all $Q$ within the 
domain of interest, $h$ is $\alpha$-strongly convex with respect to the weighted $L_p$-norm defined in 
Equation~\eqref{eq:pnorm} with $p=2$.

Requiring that $\varphi'' > \alpha$ hold uniformly is clearly too strong of a condition, as any divergence satisfying 
this condition can be seen to be lower bounded by $\alpha\cdot\mathcal{\Dw}_{\chi^2}\pa{\cdot\middle\|Q_0}$. Thus, the 
best generalization bound that our main theorem implies for such choices of $\varphi$ is the one 
stated for $p=2$ in Corollary~\ref{cor:pnorm}. Alternatively, strong convexity can hold uniformly over 
the 
domain if we can ensure that for all $P\in\Delta_n$ and all data sets $s$, $\frac{\dd P_{|s}}{\dd Q_0}$ is bounded 
within an interval $(m,M) \subset (0,\infty)$ and $\varphi'' > 0$. We refer to Table~1 in \citet{Mel20} that presents 
the strong convexity constants that can be derived using this method for a range of $f$-divergences including the 
squared Hellinger distance, the reverse relative entropy $\DDKL{Q_0}{Q}$, the Vincze--Le Cam distance, or the 
Jensen--Shannon divergence. Since all of these are of the order $M^{-c}$ for some $c>1$, we do not deem these 
divergences particularly interesting, due to the rather unrealistic assumption that $M$ be small. That said, we find it 
plausible that one can derive meaningful strong convexity properties of $f$-divergences in terms of norms 
other than the $L_2$ norm.

As a concrete example, consider the squared Hellinger divergence defined via $\varphi(x)
= \pa{\sqrt{x} - 1}^2$: 
\[
 \mathcal{D}_{\mathcal{H}}\pa{Q\|Q_0} = \int_{\Ww} \pa{\sqrt{\frac{\dd Q}{\dd Q_0}} - 1}^2 \dd Q_0 = 
 \int_{\Ww} \pa{\sqrt{\frac{\dd Q}{\dd \nu}} - \sqrt{\frac{\dd Q_0}{\dd \nu}}}^2 \dd \nu,
\]
where $\nu$ can be chosen as an arbitrary measure that dominates both $Q$ and $Q_0$.
The first derivative of $\varphi$ is $\varphi'(x) = 1 - \frac{1}{\sqrt{x}}$ and 
the second derivative is $f''(x) = x^{-3/2}/2$. Thus, in order to guarantee strong convexity with respect to 
$\norm{\cdot}_{2,Q_0}$, one needs to ensure that $\frac{\dd Q}{\dd Q_0}$ is upper-bounded by $M$, which results in a 
strong-convexity constant of $M^{-3/2}/2$.

\subsection{Bregman divergences}
Another possibility is to use Bregman divergences of appropriately defined convex functions of $Q$. To be specific, we 
consider a twice-differentiable convex function $\psi:\real_+\ra\real$ and a measure $\nu$  that dominates all 
distributions $Q\in\ev{P_{|s}: s\in\Sw}$ and define 
\[
  \mathcal{D}_{\psi}\pa{Q\|Q_0} = \int_\Ww \psi\pa{\frac{\dd Q}{\dd \nu}} \dd \nu 
  - \int_\Ww \psi\pa{\frac{\dd Q_0}{\dd \nu}} \dd \nu + \int_\Ww \psi'\pa{\frac{\dd Q_0}{\dd \nu}} 
\pa{\frac{\dd Q_0}{\dd \nu} - \frac{\dd Q}{\dd \nu}} \dd \nu,
 \]
 which is the Bregman divergence associated with the function $\int_{\Ww} \psi\pa{\frac{\dd Q}{\dd \nu}} \dd \nu$.  
Among the previously discussed divergences, the relative entropy and the $\chi^2$ divergences can be also written as 
Bregman divergences, with the special choice $\nu = Q_0$. 

In the general case, we can extend the Taylor expansion argument of Equation~\eqref{eq:f_expansion} to see 
that Bregman divergences can also satisfy a strong convexity property in terms of the $L_2$ norm $\norm{\cdot}_{2,\nu}$ 
as long as $\psi''\pa{\frac{\dd Q}{\dd \nu}}$ is uniformly bounded away from zero for all $Q$. Once again, this is a 
quite restrictive condition that can only be warranted if $\frac{\dd Q}{\dd \nu}$ is uniformly small. 
This is satisfied, for instance, when $\Ww$ is countable and $\nu$ is the counting measure so that $\frac{\dd Q}{\dd 
\nu} \le 1$. This comes at the severe price of the divergences taking enormous values that can be proportional to the 
size of the domain.

As an illustration, consider the Bregman divergence induced by $\psi(x) = -\log x$,  known as 
the Itakura--Saito divergence:
\[
 \mathcal{D}_{\mathrm{IS}}\pa{Q\middle\|Q_0} = 
\int_{\Ww} \pa{\frac{\dd Q / \dd \nu}{\dd Q_0 / \dd \nu} - \log\pa{\frac{\dd Q / \dd \nu}{\dd Q_0 / \dd \nu}} - 
1} \dd \nu.
\]
The second derivative of this function is $\psi(x) = x^{-2}$, which implies that it is $1$-strongly convex with respect 
to $\norm{\cdot}_{2,\nu}$. While this may seem like a positive result, it is overshadowed by the possibility that the 
divergence itself can grow linearly with the size of the domain $\Ww$.

\end{document}